\documentclass{article}

\usepackage{times}
\usepackage{graphicx} 
\usepackage{subfigure}

\usepackage{natbib}

\usepackage{algorithm}
\usepackage{algorithmic}
\usepackage[algo2e,linesnumbered,lined,boxed,vlined]{algorithm2e}

\usepackage{amsfonts}
\usepackage{amssymb}

\usepackage{hyperref}

\usepackage[accepted]{icml2016}

\usepackage{comment}
\usepackage{amsthm}
\usepackage{mathtools}
\usepackage{tabularx}
\usepackage{multirow}

\newtheorem{lemma}{Lemma}
\newtheorem{proposition}{Proposition}
\newtheorem{corollary}{Corollary}

\def\b{\ensuremath\boldsymbol}

\usepackage{lastpage}
\usepackage{fancyhdr}
\usepackage{url}
\icmltitlerunning{Restricted Boltzmann Machine and Deep Belief Network: Tutorial and Survey}

\begin{document}

\AddToShipoutPictureBG*{%
  \AtPageUpperLeft{%
    \setlength\unitlength{1in}%
    \hspace*{\dimexpr0.5\paperwidth\relax}
    \makebox(0,-0.75)[c]{\normalsize {\color{black} To appear as a part of an upcoming textbook on dimensionality reduction and manifold learning.}}
    }}

\twocolumn[
\icmltitle{Restricted Boltzmann Machine and Deep Belief Network: Tutorial and Survey}

\icmlauthor{Benyamin Ghojogh}{bghojogh@uwaterloo.ca}
\icmladdress{Department of Electrical and Computer Engineering, 
\\Machine Learning Laboratory, University of Waterloo, Waterloo, ON, Canada}
\icmlauthor{Ali Ghodsi}{ali.ghodsi@uwaterloo.ca}
\icmladdress{Department of Statistics and Actuarial Science \& David R. Cheriton School of Computer Science, 
\\Data Analytics Laboratory, University of Waterloo, Waterloo, ON, Canada}
\icmlauthor{Fakhri Karray}{karray@uwaterloo.ca}
\icmladdress{Department of Electrical and Computer Engineering, 
\\Centre for Pattern Analysis and Machine Intelligence, University of Waterloo, Waterloo, ON, Canada}
\icmlauthor{Mark Crowley}{mcrowley@uwaterloo.ca}
\icmladdress{Department of Electrical and Computer Engineering, 
\\Machine Learning Laboratory, University of Waterloo, Waterloo, ON, Canada}

\icmlkeywords{Tutorial}

\vskip 0.3in
]

\begin{abstract}
This is a tutorial and survey paper on Boltzmann Machine (BM), Restricted Boltzmann Machine (RBM), and Deep Belief Network (DBN). We start with the required background on probabilistic graphical models, Markov random field, Gibbs sampling, statistical physics, Ising model, and the Hopfield network. Then, we introduce the structures of BM and RBM. The conditional distributions of visible and hidden variables, Gibbs sampling in RBM for generating variables, training BM and RBM by maximum likelihood estimation, and contrastive divergence are explained. Then, we discuss different possible discrete and continuous distributions for the variables. We introduce conditional RBM and how it is trained. Finally, we explain deep belief network as a stack of RBM models. This paper on Boltzmann machines can be useful in various fields including data science, statistics, neural computation, and statistical physics. 
\end{abstract}

\section{Introduction}

Centuries ago, the Boltzmann distribution \cite{boltzmann1868studien}, also called the Gibbs distribution \cite{gibbs1902elementary}, was proposed. This energy-based distribution was found to be useful for modeling the physical systems statistically \cite{huang1987statistical}. 
One of these systems was the Ising model which modeled interacting particles with binary spins \cite{lenz1920beitrvsge,ising1925beitrag}. Later, the Ising model was found to be able to be a neural network \cite{little1974existence}. Hence, Hopfield network was proposed which modeled an Ising model in a network for modeling memory \cite{hopfield1982neural}. 
Inspired by the Hopfield network \cite{little1974existence,hopfield1982neural}, which was itself inspired by the physical Ising model \cite{lenz1920beitrvsge,ising1925beitrag}, Hinton et. al. proposed Boltzmann Machine (BM) and Restricted Boltzmann Machine (RBM) \cite{hinton1983optimal,ackley1985learning}. These models are energy-based models \cite{lecun2006tutorial} and the names come from the Boltzmann distribution \cite{boltzmann1868studien,gibbs1902elementary} used in these models. 
A BM has weighted links between two layers of neurons as well as links between neurons of every layer. RBM restricts these links to not have links between neurons of a layer. 
BM and RBM take one of the layers as the layer of data and the other layer as a representation or embedding of data. 
BM and RBM are special cases of the Ising model whose weights (coupling parameters) are learned. BM and RBM are also special cases of the Hopfield network whose weights are learned by maximum likelihood estimation rather than the Hebbian learning method \cite{hebb1949organization} which is used in the Hopfield network.

The Hebbian learning method, used in the Hopfield network, was very weak and could not generalize well to unseen data. Therefore, backpropagation \cite{rumelhart1986learning} was proposed for training neural networks.
Backpropagation was gradient descent plus the chain rule technique. 
However, researchers found out soon that neural networks cannot get deep in their number of layers. This is because, in deep networks, gradients become very small in the initial layers after many chain rules from the last layers of network. This problem was called vanishing gradients. This problem of networks plus the glory of theory in kernel support vector machines \cite{boser1992training} resulted in the winter of neural networks in the last years of previous century until around 2006.  

During the winter of neural networks, Hinton tried to save neural networks from being forgotten in the history of machine learning. So, he returned to his previously proposed RBM and proposed a learning method for RBM with the help of some other researchers including Max Welling \cite{hinton2002training,welling2004exponential}. 
They proposed training the weights of BM and RBM using maximum likelihood estimation. BM and RBM can be seen as generative models where new values for neurons can be generated using Gibbs sampling \cite{geman1984stochastic}. 
Hinton noticed RBM because he knew that the set of weights between every two layers of a neural network is an RBM. 
It was in the year 2006 \cite{hinton2006reducing,hinton2006fast} that he thought it is possible to train a network in a greedy way\footnote{A greedy algorithm makes every decision based on the most benefit at the current step and does not care about the final outcome at the final step. This greedy approach hopes that the final step will obtain a good result by small best steps based on their current benefits.} \cite{bengio2007greedy} where the weights of every layer of network is trained using RBM training. 
This stack of RBM models with a greedy algorithm for training was named Deep Belief Network (DBN) \cite{hinton2006fast,hinton2009deep}. 
DBN allowed the networks to become deep by preparing a good initialization of weights (using RBM training) for backpropagation. This good starting point for backpropagation optimization did not face the problem of vanishing gradients anymore. 
Since the breakthrough in 2006 \cite{hinton2006reducing}, the winter of neural networks started to end gradually because the networks could get deep to become more nonlinear and handle more nonlinear data. 

DBN was used in different applications including speech recognition \cite{mohamed2009deep,mohamed2010phone,mohamed2011acoustic} and action recognition \cite{taylor2007modeling}.
Hinton was very excited about the success of RBM and was thinking that the future of neural networks belongs to DBN. However, two important techniques were proposed, which were the ReLU activation function \cite{glorot2011deep} and the dropout technique \cite{srivastava2014dropout}. These two regularization methods prevented overfitting \cite{ghojogh2019theory} and resolved vanishing gradients even without RBM pre-training. Hence, backpropagation could be used alone if the new regularization methods were utilized. The success of neural networks was found out more \cite{lecun2015deep} by its various applications, for example in image recognition \cite{karpathy2015deep}.

This is a tutorial and survey paper on BM, RBM, and DBN. 
The remainder of this paper is as follows. We briefly review the required background on probabilistic graphical models, Markov random field, Gibbs sampling, statistical physics, Ising model, and the Hopfield network in Section \ref{section_background}. The structure of BM and RBM, Gibbs sampling in RBM, training RBM by maximum likelihood estimation, contrastive divergence, and training BM are explained in Section \ref{section_RBM}. Then, we introduce different cases of states for units in RBM in Section \ref{section_distributions_of_variables}. Conditional RBM and DBN are explained in Sections \ref{section_conditional_RBM} and \ref{section_DBN}, respectively. 
Finally, Section \ref{section_conclusion} concludes the paper.

\section*{Required Background for the Reader}

This paper assumes that the reader has general knowledge of calculus, probability, linear algebra, and basics of optimization. 
The required background on statistical physics is explained in the paper. 

\section{Background}\label{section_background}

\subsection{Probabilistic Graphical Model and Markov Random Field}


A Probabilistic Graphical Model (PGM) is a graph-based representation of a complex distribution in a possibly high dimensional space \cite{koller2009probabilistic}. 
In other words, PGM is a combination of graph theory and probability theory. 
In a PGM, the random variables are represented by nodes or vertices. There exist edges between two variables which have interaction with one another in terms of probability. Different conditional probabilities can be represented by a PGM. 
There exist two types of PGM which are Markov network (also called Markov random field) and Bayesian network \cite{koller2009probabilistic}. In the Markov network and Bayesian network, the edges of graph are undirected and directed, respectively. 
BM and RBM are Markov networks (Markov random field) because their links are undirected \cite{hinton2007boltzmann}. 

\subsection{Gibbs Sampling}\label{section_Gibbs_sampling_background}

Gibbs sampling, firstly proposed by \cite{geman1984stochastic}, draws samples from a $d$-dimensional multivariate distribution $\mathbb{P}(X)$ using $d$ conditional distributions  \cite{bishop2006pattern}. 
This sampling algorithm assumes that the conditional distributions of every dimension of data conditioned on the rest of coordinates are simple to draw samples from.

In Gibbs sampling, we desire to sample from a multivariate distribution $\mathbb{P}(X)$ where $X \in \mathbb{R}^d$. Consider the notation $\mathbb{R}^d \ni \b{x} := [x_1, x_2, \dots, x_d]^\top$.
We start from a random $d$-dimensional vector in the range of data. Then, we sample the first dimension of the first sample from the distribution of the first dimension conditioned on the other dimensions. We do it for all dimensions, where the $j$-th dimension is sampled as \cite{ghojogh2020sampling}:
\begin{align}
x_j \sim \mathbb{P}(x_j\, |\, x_1, \dots, x_{j-1}, x_{j+1}, \dots, x_d).
\end{align}
We do this for all dimensions until all dimensions of the first sample are drawn. Then, starting from the first sample, we repeat this procedure for the dimensions of the second sample. We iteratively perform this for all samples; however, some initial samples are not yet valid because the algorithm has started from a not-necessarily valid vector. We accept all samples after some burn-in iterations. 
Gibbs sampling can be seen as a special case of the Metropolis-Hastings algorithm which accepts the proposed samples with probability one (see \cite{ghojogh2020sampling} for proof).
Gibbs sampling is used in BM and RBM for generating visible and hidden samples. 

\subsection{Statistical Physics and Ising Model}

\subsubsection{Boltzmann (Gibbs) Distribution}

Assume we have several particles $\{x_i\}_{i=1}^d$ in statistical physics. These particles can be seen as randome variables which can randomly have a state. For example, if the particles are electrons, they can have states $+1$ and $-1$ for counterclockwise and clockwise spins, respectively. 
The \textit{Boltzmann distribution} \cite{boltzmann1868studien}, also called the \textit{Gibbs distribution} \cite{gibbs1902elementary}, can show the probability that a physical system can have a specific state. i.e., every of the particles has a specific state. The probability mass function of this distribution is \cite{huang1987statistical}:
\begin{align}\label{equation_Boltmann_distribution}
\mathbb{P}(x) = \frac{e^{-\beta E(x)}}{Z},
\end{align}
where $E(x)$ is the energy of variable $x$ and $Z$ is the normalization constant so that the probabilities sum to one. This normalization constant is called the \textit{partition function} which is hard to compute as it sums over all possible configurations of states (values) that the particles can have. If we define $\mathbb{R}^d \ni \b{x} := [x_1, \dots, x_d]^\top$, we have:
\begin{align}\label{equation_partition_function_of_Boltzmann_distribution}
Z := \sum_{\b{x} \in \mathbb{R}^d} e^{-\beta E(x)}.
\end{align}
The coefficient $\beta \geq 0$ is defined as:
\begin{align}
\beta := \frac{1}{k_\beta T} \propto \frac{1}{T},
\end{align}
where $k_\beta$ is the Boltzmann constant and $T \geq 0$ is the absolute thermodynamic temperature in Kelvins. 
If the temperature tends to absolute zero, $T \rightarrow 0$, we have $\beta \rightarrow \infty$ and $\mathbb{P}(x) \rightarrow 0$, meaning that the absolute zero temperature occurs extremely rarely in the universe. 

The \textit{free energy} is defined as:
\begin{align}
F(\beta) := \frac{-1}{\beta} \ln(Z),
\end{align}
where $\ln(.)$ is the natural logarithm. 
The \textit{internal energy} is defined as:
\begin{align}
U(\beta) := \frac{\partial}{\partial \beta} \big(\beta\, F(\beta)\big).
\end{align}
Therefore, we have:
\begin{align}
U(\beta) &= \frac{\partial}{\partial \beta} (- \ln(Z)) = \frac{-1}{Z} \frac{\partial Z}{\partial \beta} \nonumber \\
&\overset{(\ref{equation_partition_function_of_Boltzmann_distribution})}{=} \sum_{\b{x} \in \mathbb{R}^d} E(x)\, \frac{e^{-\beta E(x)}}{Z} \overset{(\ref{equation_Boltmann_distribution})}{=} \sum_{\b{x} \in \mathbb{R}^d} \mathbb{P}(x) E(x). \label{equation_internal_energy_relatedTo_energy}
\end{align}
The \textit{entropy} is defined as:
\begin{align}
H(\beta) &:= -\sum_{\b{x} \in \mathbb{R}^d} \mathbb{P}(x) \ln\big(\mathbb{P}(x)\big) \nonumber\\
&\overset{(\ref{equation_Boltmann_distribution})}{=} -\sum_{\b{x} \in \mathbb{R}^d} \mathbb{P}(x) \big(-\beta E(x) - \ln(Z)\big) \nonumber\\
&= \beta \sum_{\b{x} \in \mathbb{R}^d} \mathbb{P}(x) E(x) + \ln(Z) \underbrace{\sum_{\b{x} \in \mathbb{R}^d} \mathbb{P}(x)}_{=1} \nonumber\\
&\overset{(a)}{=} -\beta\, F(\beta) + \beta\, U(\beta),
\end{align}
where $(a)$ is because of Eqs. (\ref{equation_internal_energy_relatedTo_energy}) and (\ref{equation_partition_function_of_Boltzmann_distribution}).

\begin{lemma}\label{lemma_tend_to_low_energy}
A physical system prefers to be in low energy; hence, the system always loses energy to have less energy. 
\end{lemma}
\begin{proof}
On the one hand, according to the second law of thermodynamics, entropy of a physical system always increases by passing time \cite{carroll2010eternity}. Entropy is a measure of randomness and disorder in system. On the other hand, when a system loses energy to its surrounding, it becomes less ordered. Hence, by passing time, the energy of system decreases to have more entropy. Q.E.D.
\end{proof}

\begin{corollary}
According to Eq. (\ref{equation_Boltmann_distribution}) and Lemma \ref{lemma_tend_to_low_energy}, the probability $\mathbb{P}(x)$ of states in a system tend to increase by passing time. 
\end{corollary}
This corollary makes sense because systems tend to become more probable. This idea is also used in simulated annealing\footnote{Simulated annealing is a metaheuristic optimization algorithm in which a temperature parameter controls the amount of global search versus local search. It reduces the temperature gradually to decrease the exploration and increase the exploitation of the search space, gradually.} \cite{kirkpatrick1983optimization} where the temperature of system is cooled down gradually. Simulated annealing and temperature-based learning have been used in BM models \cite{passos2018temperature,alberici2020annealing,alberici2021deep}. 

\subsubsection{Ising Model}


The Ising model \cite{lenz1920beitrvsge,ising1925beitrag}, also known as the Lenz-Ising model, is a model in which the particles can have $-1$ or $+1$ spins \cite{brush1967history}. Therefore, $x_i \in \{-1, +1\}, \forall i \in \{1, \dots, d\}$. It uses the Boltzmann distribution, Eq. (\ref{equation_Boltmann_distribution}), where the energy function is defined as:
\begin{align}\label{equation_energy_in_Ising_model}
E(x) := \mathcal{H}(x) = -\sum_{(i,j)} J_{ij}\, x_i\, x_j,
\end{align}
where $\mathcal{H}(x)$ is called the Hamiltonian, $J_{ij} \in \mathbb{R}$ is the coupling parameter, and the summation is over particles which interact with each other. Note that as energy is proportional to the reciprocal of squared distance, nearby particles are only assumed to be interacting. Therefore, usually the interaction graph of particles is a chain (one dimensional grid), mesh grid (lattice), closed chain (loop), or torus (multi-dimensional loop). 

Based on the characteristic of model, the coupling parameter has different values. If for all interacting $i$ and $j$, we have $J_{ij} \geq 0$ or $J_{ij} < 0$, the model is named \textit{ferromagnetic} and \textit{anti-ferromagnetic}, respectively. If $J_{ij}$ can be both positive and negative, the model is called a \textit{spin glass}. If the coupling parameters are all constant, the model is \textit{homogeneous}.
According to Lemma \ref{lemma_tend_to_low_energy}, the energy decreases over time. According to Eq. (\ref{equation_energy_in_Ising_model}), in ferromagnetic model ($J_{ij} \geq 0$), the energy of Ising model decreases if the interacting $x_i$ and $x_j$ have the same state (spin) because of the negative sign before summation. Likewise, in anti-ferromagnetic models, the nearby particles tend to have different spins over time. 

According to Eq. (\ref{equation_energy_in_Ising_model}), in ferromagnetic models, the energy is zero if $J_{ij} = 0$. This results in $\mathbb{P}(x) = 1$ according to Eq. (\ref{equation_Boltmann_distribution}). With similar analysis and according to the previous discussion, in ferromagnetic models, $J_{ij} \rightarrow \infty$ yields to having the same spins for all particles. This means that we finally have all $+1$ spins with probability of half or all $-1$ spins with probability of half. 
Ising models can be modeled as normal factor graphs. For more information on this, refer to \cite{molkaraie2017primal,molkaraie2020marginal}. 
The BM and RBM are Ising models whose coupling parameters are considered as weights and these weights are learned using maximum likelihood estimation \cite{hinton2007boltzmann}. Hence, we can say that BM and RBM are energy-based learning methods \cite{lecun2006tutorial}.

\subsection{Hopfield Network}\label{section_background_Hopfield}

It was proposed in \cite{little1974existence} to use the Ising model in a neural network structure. Hopfield extended this idea to model the memory by a neural network. The resulting network was the Hopfield network \cite{hopfield1982neural}. 
This network has some units or neurons denoted by $\{x_i\}_{i=1}^d$.
The states or outputs of units are all binary $x_i \in \{-1,+1\}, \forall i$. Let $w_{ij}$ denote the weight of link connecting unit $i$ to unit $j$. The weights of Hopfield network are learned using the Hebbian learning (Hebb's law of association) \cite{hebb1949organization}:
\begin{align}
w_{ij} := 
\left\{
    \begin{array}{ll}
        x_i \times x_j & \mbox{if } i \neq j, \\
        0 & \mbox{otherwise.}
    \end{array}
\right.
\end{align}
After training, the outputs of units can be determined for an input if the weighted summation of inputs to unit passes a threshold $\theta$:
\begin{align}
x_i := 
\left\{
    \begin{array}{ll}
        +1 & \mbox{if } \sum_{j=1}^d w_{ij} x_j \geq \theta, \\
        -1 & \mbox{otherwise.}
    \end{array}
\right.
\end{align}
In the original paper of Hopfield network \cite{hopfield1982neural}, the binary states are $x_i \in \{0,1\}, \forall i$ so the Hebbian learning is $w_{ij} := (2 x_i - 1) \times (2 x_j - 1), \forall i \neq j$. 
Hopfield network is an Ising model so it uses Eq. (\ref{equation_energy_in_Ising_model}) as its energy. This energy is also used in the Boltzmann distribution which is Eq. (\ref{equation_Boltmann_distribution}).

It is noteworthy that there are also Hopfield networks with continuous states \cite{hopfield1984neurons}. 
Modern Hopfield networks, such as \cite{ramsauer2020hopfield}, are often based on dense associative memories \cite{krotov2016dense}. 
Some other recent works on associative memories are \cite{krotov2021large,krotov2021hierarchical}. 
The BM and RBM models are Hopfield networks whose weights are learned using maximum likelihood estimation and not Hebbian learning. 

\section{Restricted Boltzmann Machine}\label{section_RBM}





\subsection{Structure of Restricted Boltzmann Machine}\label{section_RBM_structure}

\begin{figure*}[!t]
\centering
\includegraphics[width=5in]{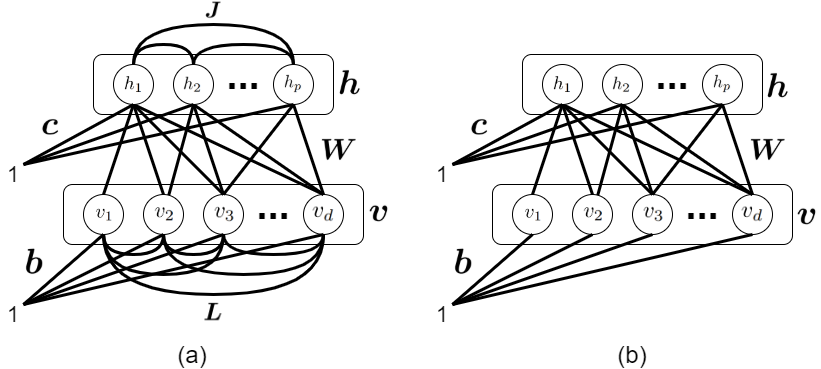}
\caption{The structures of (a) a Boltzmann machine and (b) a restricted Boltzmann machine.}
\label{figure_BM_and_RBM}
\end{figure*}

Boltzmann Machine (BM) is a generative model and a Probabilistic Graphical Model (PGM) \cite{bishop2006pattern} which is a building block of many probabilistic models. 
Its name is because of the Boltzmann distribution \cite{boltzmann1868studien,gibbs1902elementary} used in this model. 
It was first introduced to be used in machine learning in \cite{hinton1983optimal,ackley1985learning} and then in \cite{hinton2002training,welling2004exponential}.
A BM consists of a visible (or observation) layer $\b{v} = [v_1, \dots, v_d] \in \mathbb{R}^d$ and a hidden layer $\b{h} = [h_1, \dots, h_p] \in \mathbb{R}^p$. The visible layer is the layer that we can see; for example, it can be the layer of data. The hidden layer is the layer of latent variables which represent meaningful features or embeddings for the visible data. In other words, there is a meaningful connection between the hidden and visible layers although their dimensionality might differ, i.e., $d \neq p$. In the PGM of BM, there are connection links between elements of $\b{v}$ and elements of $\b{h}$. Each of the elements of $\b{v}$ and $\b{h}$ also have a bias. There are also links between the elements of $\b{v}$ as well as between the elements of $\b{h}$ \cite{salakhutdinov2009deep}. Let $w_{ij}$ denote the link between $v_i$ and $h_j$, and $l_{ij}$ be the link between $v_i$ and $v_j$, and $j_{ij}$ be the link between $h_i$ and $h_j$, and $b_i$ be the bias link for $v_i$, and $c_i$ be the bias link for $h_i$. The dimensionality of these links are $\b{W} = [w_{ij}] \in \mathbb{R}^{d \times p}$, $\b{L} = [l_{ij}] \in \mathbb{R}^{d \times d}$, $\b{J} = [j_{ij}] \in \mathbb{R}^{p \times p}$, $\b{b} = [b_1, \dots, b_d] \in \mathbb{R}^d$, and $\b{c} = [c_1, \dots, c_p] \in \mathbb{R}^p$. 
Note that $\b{W}$ is a symmetric matrix, i.e., $w_{ij} = w_{ji}$. Also, as there is no link from a node to itself, the diagonal elements of $\b{L}$ and $\b{J}$ are zero, i.e., $l_{ii} = j_{ii} = 0, \forall i$.
Restricted Boltzmann Machine (RBM) is BM which does not have links within a layer, i.e., there is no any link between the elements of $\b{v}$ and no any link between the elements of $\b{h}$. In other words, the links are restricted in RBM to be $\b{L} = \b{J} = \b{0}$. 
Figure \ref{figure_BM_and_RBM} depicts both BM and RBM with their layers and links. 
In this section, we focus on RBM. 

Recall that RBM is an Ising model.
As we saw in Eq. (\ref{equation_energy_in_Ising_model}), the energy of an Ising model can be modeled as \cite{hinton1983optimal,ackley1985learning}: 
\begin{align}\label{equation_energy_function}
\mathbb{R} \ni E(\b{v}, \b{h}) := -\b{b}^\top \b{v} - \b{c}^\top \b{h} - \b{v}^\top \b{W} \b{h},
\end{align}
which is based on interactions between linked units. 
As introduced in Eq. (\ref{equation_Boltmann_distribution}), the visible and hidden variables make a joint Boltzmann distribution \cite{hinton2012practical}:
\begin{equation}\label{equation_joint_prob_v_and_h}
\begin{aligned}
\mathbb{P}(\b{v}, \b{h}) &= \frac{1}{Z} \exp(-E(\b{v}, \b{h})) \\
&\overset{(\ref{equation_energy_function})}{=} \frac{1}{Z} \exp(\b{b}^\top \b{v} + \b{c}^\top \b{h} + \b{v}^\top \b{W} \b{h}),
\end{aligned}
\end{equation}
where $\b{Z}$ is the partition function:
\begin{align}\label{equation_partition_function}
Z := \sum_{\b{v} \in \mathbb{R}^d} \sum_{\b{h} \in \mathbb{R}^p} \exp(-E(\b{v}, \b{h})).
\end{align}

According to Lemma \ref{lemma_tend_to_low_energy}, the BM and RBM try to reduce the energy of model. Training the BM or RBM reduces its energy \cite{hinton1983optimal,ackley1985learning}.

\subsection{Conditional Distributions}

\begin{proposition}[Conditional Independence of Variables]\label{proposition_RBM_conditional_independence}
In RBM, given the visible variables, the hidden variables are conditionally independent. Likewise, given the hidden variables, the visible variables are conditionally independent. This does not hold in BM because of the links within each layer.
\end{proposition}
\begin{proof}
According to the Bayes' rule, we have:
\begin{align*}
&\mathbb{P}(\b{h} | \b{v}) = \frac{\mathbb{P}(\b{h}, \b{v})}{\mathbb{P}(\b{v})} = \frac{\mathbb{P}(\b{v}, \b{h})}{\sum_{\b{h} \in \mathbb{R}^p} \mathbb{P}(\b{v}, \b{h})} \\
&\overset{(\ref{equation_joint_prob_v_and_h})}{=} \frac{\frac{1}{Z} \exp(\b{b}^\top \b{v} + \b{c}^\top \b{h} + \b{v}^\top \b{W} \b{h})}{\sum_{\b{h} \in \mathbb{R}^p} \frac{1}{Z} \exp(\b{b}^\top \b{v} + \b{c}^\top \b{h} + \b{v}^\top \b{W} \b{h})} \\
&= \frac{\frac{1}{Z} \exp(\b{b}^\top \b{v}) \exp(\b{c}^\top \b{h}) \exp(\b{v}^\top \b{W} \b{h})}{\frac{1}{Z} \sum_{\b{h} \in \mathbb{R}^p} \exp(\b{b}^\top \b{v}) \exp(\b{c}^\top \b{h}) \exp(\b{v}^\top \b{W} \b{h})} 
\end{align*}
\begin{align*}
&\overset{(a)}{=} \frac{\exp(\b{b}^\top \b{v}) \exp(\b{c}^\top \b{h}) \exp(\b{v}^\top \b{W} \b{h})}{\exp(\b{b}^\top \b{v}) \sum_{\b{h} \in \mathbb{R}^p} \exp(\b{c}^\top \b{h}) \exp(\b{v}^\top \b{W} \b{h})} \\
&= \frac{\exp(\b{c}^\top \b{h}) \exp(\b{v}^\top \b{W} \b{h})}{\sum_{\b{h} \in \mathbb{R}^p} \exp(\b{c}^\top \b{h}) \exp(\b{v}^\top \b{W} \b{h})},
\end{align*}
where $(a)$ is because the term $\exp(\b{b}^\top \b{v})$ does not have $\b{h}$ in it. 
Note that $\sum_{\b{h} \in \mathbb{R}^p}$ denotes summation over all possible $p$-dimensional hidden variables for the sake of marginalization.
Let $Z' := \sum_{\b{h} \in \mathbb{R}^p} \exp(\b{c}^\top \b{h}) \exp(\b{v}^\top \b{W} \b{h})$. Hence:
\begin{align}
\mathbb{P}(\b{h} | \b{v}) &= \frac{1}{Z'} \exp(\b{c}^\top \b{h} + \b{v}^\top \b{W} \b{h}) \nonumber \\
&= \frac{1}{Z'} \exp\Big(\sum_{j=1}^p c_j h_j + \sum_{j=1}^p \b{v}^\top \b{W}_{:j} h_j\Big) \nonumber \\
&= \frac{1}{Z'} \prod_{j=1}^p \exp( c_j h_j + \b{v}^\top \b{W}_{:j} h_j), \label{equation_prob_h_given_v} 
\end{align}
where $\b{W}_{:j} \in \mathbb{R}^d$ denotes the $j$-th column of matrix $\b{W}$.
The Eq. (\ref{equation_prob_h_given_v}) shows that given the visible variables, the hidden variables are conditionally independent because their joint distribution is the product of every distribution. We can write similar expressions for the probability $\mathbb{P}(\b{v} | \b{h})$:
\begin{align}
\mathbb{P}(\b{v} | \b{h}) &= \frac{1}{Z''} \exp(\b{b}^\top \b{v} + \b{v}^\top \b{W} \b{h}) \nonumber \\
&= \frac{1}{Z''} \exp\Big(\sum_{i=1}^d b_i v_i + \sum_{i=1}^d v_i \b{W}_{i:} \b{h}\Big) \nonumber \\
&= \frac{1}{Z''} \prod_{i=1}^d \exp( b_i v_i + v_i \b{W}_{i:} \b{h}), \label{equation_prob_v_given_h} 
\end{align}
where $\b{W}_{i:} \in \mathbb{R}^p$ denotes the $i$-th row of matrix $\b{W}$ and $Z'' := \sum_{\b{v} \in \mathbb{R}^d} \exp(\b{b}^\top \b{v}) \exp(\b{v}^\top \b{W} \b{h})$.
This equation shows that given the hidden variables, the visible variables are conditionally independent. Q.E.D.
\end{proof}

According to Eq. (\ref{equation_prob_h_given_v}) and considering the rule $\mathbb{P}(\b{h} | \b{v}) = \mathbb{P}(\b{h}, \b{v}) / \mathbb{P}(\b{v})$, we have:
\begin{align}
&\mathbb{P}(\b{h} | \b{v}) = \frac{1}{Z'}  \prod_{j=1}^p \exp( c_j h_j + \b{v}^\top \b{W}_{:j} h_j) \nonumber \\
&~~~~~~~~~~~~ = \frac{1}{Z'} \prod_{j=1}^p \mathbb{P}(h_j, \b{v}) \nonumber 
\end{align}
\begin{equation}\label{equation_prob_h_j_and_v}
\begin{aligned}
\implies \mathbb{P}(h_j, \b{v}) &= \exp( c_j h_j + \b{v}^\top \b{W}_{:j} h_j) \\
&= \exp( c_j h_j + \sum_{i=1}^d v_i w_{ij} h_j). 
\end{aligned}
\end{equation}
Similarly, according to Eq. (\ref{equation_prob_v_given_h}) and considering the rule $\mathbb{P}(\b{v} | \b{h}) = \mathbb{P}(\b{h}, \b{v}) / \mathbb{P}(\b{h})$, we have:
\begin{align}
&\mathbb{P}(\b{h} | \b{v}) = \frac{1}{Z''}  \prod_{i=1}^d \exp( b_i v_i + v_i \b{W}_{i:}\, \b{h}) \nonumber \\
&~~~~~~~~~~~~ = \frac{1}{Z''} \prod_{i=1}^d \mathbb{P}(\b{h}, v_i) \nonumber 
\end{align}
\begin{equation}\label{equation_prob_h_and_v_i}
\begin{aligned}
\implies \mathbb{P}(\b{h}, v_i) &= \exp( b_i v_i + v_i \b{W}_{i:}\, \b{h}) \\
&= \exp( b_i v_i + \sum_{j=1}^p v_i w_{ij} h_j). 
\end{aligned}
\end{equation}

We will use these equations later.

\subsection{Sampling Hidden and Visible Variables}\label{section_RBM_Gibbs_sampling}

\subsubsection{Gibbs Sampling}

We can use Gibbs sampling for sampling and generating the hidden and visible units. If $\nu$ denotes the iteration index of Gibbs sampling, we iteratively sample:
\begin{align}
& \b{h}^{(\nu)} \sim \mathbb{P}(\b{h} | \b{v}^{(\nu)}), \\
& \b{v}^{(\nu+1)} \sim \mathbb{P}(\b{v} | \b{h}^{(\nu)}),
\end{align}
until the burn-in convergence. As was explained in Section \ref{section_Gibbs_sampling_background}, only several iterations of Gibbs sampling are usually sufficient. After the burn-in, the samples are approximate samples from the joint distribution $\mathbb{P}(\b{v}, \b{h})$.
As the variables are conditionally independent, this Gibbs sampling can be implemented as in Algorithm \ref{algorithm_RBM_Gibbs_sampling}. In this algorithm, $h_j^{(\nu)} \sim \mathbb{P}(h_j | \b{v}^{(\nu)})$ can be implemented as drawing a sample from uniform distribution $u \sim U[0,1]$ and comparing it to the value of Probability Density Function (PDF), $\mathbb{P}(h_j | \b{v}^{(\nu)})$. If $u$ is less than or equal to this value, we have $h_j=1$; otherwise, we have $h_j=0$. Implementation of sampling $v_i$ has a similar procedure. Alternatively, we can use inverse of cumulative distribution function of these distributions for drawing samples (see \cite{ghojogh2020sampling} for more details about sampling).

\subsubsection{Generations and Evaluations by Gibbs Sampling}

Gibbs sampling for generating both observation and hidden units is used for both training and evaluation phases of RBM. Use of Gibbs sampling in training RBM will be explained in Sections \ref{section_RBM_training} and \ref{section_contrastive_divergence}. After the RBM model is trained, we can generate any number of $p$-dimensional hidden variables as a meaningful representation of the $d$-dimensional observation using Gibbs sampling. Moreover, using Gibbs sampling, we can generate other $d$-dimensional observations in addition to the original dataset. These new generated observations are $d$-dimensional representations for the $p$-dimensional hidden variables. This shows that BM and RBM are generative models. 

\SetAlCapSkip{0.5em}
\IncMargin{0.8em}
\begin{algorithm2e}[!t]
\DontPrintSemicolon
    \textbf{Input}: visible dataset $\b{v}$, (initialization: optional)\;
    Get initialization or do random initialization of $\b{v}$\;
    \While{until burn-in}{
        \For{$j$ from $1$ to $p$}{
            $h_j^{(\nu)} \sim \mathbb{P}(h_j | \b{v}^{(\nu)})$\;
        }
        \For{$i$ from $1$ to $d$}{
            $v_i^{(\nu+1)} \sim \mathbb{P}(v_i | \b{h}^{(\nu)})$\;
        }
    }
\caption{Gibbs sampling in RBM}\label{algorithm_RBM_Gibbs_sampling}
\end{algorithm2e}
\DecMargin{0.8em}

\subsection{Training Restricted Boltzmann Machine by Maximum Likelihood Estimation}\label{section_RBM_training}

The weights of links which are $\b{W}$, $\b{b}$, and $\b{c}$ should be learned so that we can use them for sampling/generating the hidden and visible units. 
Consider a dataset of $n$ visible vectors $\{\b{v}_i \in \mathbb{R}^d\}_{i=1}^n$.
Note that $\b{v}_i$ should not be confused with $v_i$ where the former is the $i$-th visible data instance and the latter is the $i$-th visible unit. 
We denote the $j$-th dimension of $\b{v}_i$ by $\b{v}_{i,j}$; in other words, $\b{v}_i = [\b{v}_{i,1}, \dots, \b{v}_{i,d}]^\top$.
The log-likelihood of the visible data is:
\begin{align}
&\ell(\b{W}, \b{b}, \b{c}) = \sum_{i=1}^n \log(\mathbb{P}(\b{v}_i)) = \sum_{i=1}^n \log\Big(\sum_{\b{h} \in \mathbb{R}^p} \mathbb{P}(\b{v}_i, \b{h})\Big) \nonumber\\
&\overset{(\ref{equation_joint_prob_v_and_h})}{=} \sum_{i=1}^n \log\Big(\sum_{\b{h} \in \mathbb{R}^p} \frac{1}{Z} \exp(-E(\b{v}_i, \b{h})) \Big) \nonumber\\
&= \sum_{i=1}^n \log\Big(\frac{1}{Z} \sum_{\b{h} \in \mathbb{R}^p} \exp(-E(\b{v}_i, \b{h})) \Big) \nonumber\\
&= \sum_{i=1}^n \Big[ \log\Big(\sum_{\b{h} \in \mathbb{R}^p} \exp(-E(\b{v}_i, \b{h})) \Big) - \log Z \Big] \nonumber\\
&= \sum_{i=1}^n \log\Big(\sum_{\b{h} \in \mathbb{R}^p} \exp(-E(\b{v}_i, \b{h})) \Big) - n \log Z \nonumber
\end{align}
\begin{align}
&\overset{(\ref{equation_partition_function})}{=} \sum_{i=1}^n \log\Big(\sum_{\b{h} \in \mathbb{R}^p} \exp(-E(\b{v}_i, \b{h})) \Big) \nonumber\\
&\quad\quad\quad - n \log \sum_{\b{v} \in \mathbb{R}^d} \sum_{\b{h} \in \mathbb{R}^p} \exp(-E(\b{v}, \b{h})).
\end{align}
We use Maximum Likelihood Estimation (MLE) for finding the parameters $\theta := \{\b{W}, \b{b}, \b{c}\}$.
The derivative of log-likelihood with respect to parameter $\theta$ is:
\begin{equation}
\begin{aligned}
\nabla_\theta \ell(\theta) = &\nabla_\theta \sum_{i=1}^n \log\Big(\sum_{\b{h} \in \mathbb{R}^p} \exp(-E(\b{v}_i, \b{h})) \Big) \\
&- n \nabla_\theta \log \sum_{\b{v} \in \mathbb{R}^d} \sum_{\b{h} \in \mathbb{R}^p} \exp(-E(\b{v}, \b{h})).
\end{aligned}
\end{equation}
The first term of this derivative is:
\begin{align}
&\nabla_\theta \sum_{i=1}^n \log\Big(\sum_{\b{h} \in \mathbb{R}^p} \exp(-E(\b{v}_i, \b{h})) \Big) \nonumber \\
&= \sum_{i=1}^n \nabla_\theta \log\Big(\sum_{\b{h} \in \mathbb{R}^p} \exp(-E(\b{v}_i, \b{h})) \Big) \nonumber\\
&= \sum_{i=1}^n \frac{\nabla_\theta \sum_{\b{h} \in \mathbb{R}^p} \exp(-E(\b{v}_i, \b{h}))}{\sum_{\b{h} \in \mathbb{R}^p} \exp(-E(\b{v}_i, \b{h}))} \nonumber\\
&= \sum_{i=1}^n \frac{\sum_{\b{h} \in \mathbb{R}^p} \exp(-E(\b{v}_i, \b{h})) \nabla_\theta (-E(\b{v}_i, \b{h}))}{\sum_{\b{h} \in \mathbb{R}^p} \exp(-E(\b{v}_i, \b{h}))} \nonumber\\
&\overset{(a)}{=} \sum_{i=1}^n \mathbb{E}_{\sim \mathbb{P}(\b{h}|\b{v}_i)}[\nabla_\theta (-E(\b{v}_i, \b{h}))], \label{equation_RBM_MLE_derivative_first_term}
\end{align}
where $(a)$ is because the definition of expectation is $\mathbb{E}_{\sim \mathbb{P}}[\b{x}] := \sum_{i=1} \mathbb{P}(\b{x}_i)\, \b{x}_i$. However, if $\mathbb{P}$ is not an actual distribution and does not sum to one, we should normalize it to behave like a distribution in the expectation: $\mathbb{E}_{\sim \mathbb{P}}[\b{x}] := (\sum_{i=1} \mathbb{P}(\b{x}_i)\, \b{x}_i) / (\sum_{i=1} \mathbb{P}(\b{x}_i))$.
The second term of the derivative of log-likelihood is:
\begin{align}
&- n \nabla_\theta \log \sum_{\b{v} \in \mathbb{R}^d} \sum_{\b{h} \in \mathbb{R}^p} \exp(-E(\b{v}, \b{h})) \nonumber\\
&= -n \frac{\nabla_\theta \sum_{\b{v} \in \mathbb{R}^d} \sum_{\b{h} \in \mathbb{R}^p} \exp(-E(\b{v}, \b{h}))}{\sum_{\b{v} \in \mathbb{R}^d} \sum_{\b{h} \in \mathbb{R}^p} \exp(-E(\b{v}, \b{h}))} \nonumber\\
&= -n \frac{\sum_{\b{v} \in \mathbb{R}^d} \sum_{\b{h} \in \mathbb{R}^p} \nabla_\theta \exp(-E(\b{v}, \b{h}))}{\sum_{\b{v} \in \mathbb{R}^d} \sum_{\b{h} \in \mathbb{R}^p} \exp(-E(\b{v}, \b{h}))} \nonumber\\
&= -n \frac{\sum_{\b{v} \in \mathbb{R}^d} \sum_{\b{h} \in \mathbb{R}^p} \exp(-E(\b{v}, \b{h})) \nabla_\theta (-E(\b{v}, \b{h}))}{\sum_{\b{v} \in \mathbb{R}^d} \sum_{\b{h} \in \mathbb{R}^p} \exp(-E(\b{v}, \b{h}))} \nonumber\\
&\overset{(a)}{=} -n\, \mathbb{E}_{\sim \mathbb{P}(\b{h},\b{v})}[\nabla_\theta (-E(\b{v}, \b{h}))], \label{equation_RBM_MLE_derivative_second_term}
\end{align}
where $(a)$ is for the definition of expectation which was already explained above. 
In summary, the derivative of log-likelihood is:
\begin{equation}\label{equation_RBM_loglikelihood_derivative}
\begin{aligned}
\nabla_\theta \ell(\theta) = &\sum_{i=1}^n \mathbb{E}_{\sim \mathbb{P}(\b{h}|\b{v}_i)}[\nabla_\theta (-E(\b{v}_i, \b{h}))] \\
&-n\, \mathbb{E}_{\sim \mathbb{P}(\b{h},\b{v})}[\nabla_\theta (-E(\b{v}, \b{h}))].
\end{aligned}
\end{equation}
Setting this derivative to zero does not give us a closed-form solution. Hence, we should learn the parameters iteratively using gradient ascent for MLE. 

Now, consider each of the parameters $\theta = \{\b{W}, \b{b}, \b{c}\}$. The derivative w.r.t. these parameters in Eq. (\ref{equation_RBM_loglikelihood_derivative}) are:
\begin{align*}
&\nabla_{\b{W}} (-E(\b{v}, \b{h})) \overset{(\ref{equation_energy_function})}{=} \frac{\partial}{\partial \b{W}} (\b{b}^\top \b{v} + \b{c}^\top \b{h} + \b{v}^\top \b{W} \b{h}) = \b{v} \b{h}^\top, \\
&\nabla_{\b{b}} (-E(\b{v}, \b{h})) \overset{(\ref{equation_energy_function})}{=} \frac{\partial}{\partial \b{b}} (\b{b}^\top \b{v} + \b{c}^\top \b{h} + \b{v}^\top \b{W} \b{h}) = \b{v}, \\
&\nabla_{\b{c}} (-E(\b{v}, \b{h})) \overset{(\ref{equation_energy_function})}{=} \frac{\partial}{\partial \b{c}} (\b{b}^\top \b{v} + \b{c}^\top \b{h} + \b{v}^\top \b{W} \b{h}) = \b{h}.
\end{align*}
Therefore, Eq. (\ref{equation_RBM_loglikelihood_derivative}) for these parameters becomes:
\begin{align*}
\nabla_{\b{W}} \ell(\theta) &= \sum_{i=1}^n \mathbb{E}_{\sim \mathbb{P}(\b{h}|\b{v}_i)}[\b{v} \b{h}_i^\top] -n\, \mathbb{E}_{\sim \mathbb{P}(\b{h},\b{v})}[\b{v} \b{h}^\top] \\
&= \sum_{i=1}^n \b{v}_i\, \mathbb{E}_{\sim \mathbb{P}(\b{h}|\b{v}_i)}[\b{h}^\top] -n\, \mathbb{E}_{\sim \mathbb{P}(\b{h},\b{v})}[\b{v} \b{h}^\top],
\end{align*}
\begin{align*}
\nabla_{\b{b}} \ell(\theta) &= \sum_{i=1}^n \mathbb{E}_{\sim \mathbb{P}(\b{h}|\b{v}_i)}[\b{v}_i] -n\, \mathbb{E}_{\sim \mathbb{P}(\b{h},\b{v})}[\b{v}] \\
&= \sum_{i=1}^n \b{v}_i -n\, \mathbb{E}_{\sim \mathbb{P}(\b{h},\b{v})}[\b{v}],
\end{align*}
\begin{align*}
\nabla_{\b{c}} \ell(\theta) &= \sum_{i=1}^n \mathbb{E}_{\sim \mathbb{P}(\b{h}|\b{v}_i)}[\b{h}] -n\, \mathbb{E}_{\sim \mathbb{P}(\b{h},\b{v})}[\b{h}].
\end{align*}
If we define:
\begin{align}\label{equation_h_hat}
\widehat{\b{h}}_i := \mathbb{E}_{\sim \mathbb{P}(\b{h}|\b{v}_i)}[\b{h}],
\end{align}
we can summarize these derivatives as:
\begin{align}
& \mathbb{R}^{d \times p} \ni \nabla_{\b{W}} \ell(\theta) = \sum_{i=1}^n \b{v}_i \widehat{\b{h}}_i^\top -n\, \mathbb{E}_{\sim \mathbb{P}(\b{h},\b{v})}[\b{v} \b{h}^\top], \label{equation_RBM_derivative_W} \\
& \mathbb{R}^d \ni \nabla_{\b{b}} \ell(\theta) = \sum_{i=1}^n \b{v}_i -n\, \mathbb{E}_{\sim \mathbb{P}(\b{h},\b{v})}[\b{v}], \label{equation_RBM_derivative_b} \\
& \mathbb{R}^p \ni \nabla_{\b{c}} \ell(\theta) = \sum_{i=1}^n \widehat{\b{h}}_i -n\, \mathbb{E}_{\sim \mathbb{P}(\b{h},\b{v})}[\b{h}]. \label{equation_RBM_derivative_c}
\end{align}
Setting these derivatives to zero does not give a closed form solution. Hence, we need to find the solution iteratively using gradient descent where the above gradients are used.  
In the derivatives of log-likelihood, we have two types of expectation. The conditional expectation $\mathbb{E}_{\sim \mathbb{P}(\b{h}|\b{v}_i)}[.]$ is based on the observation or data which is $\b{v}_i$. The joint expectation $\mathbb{E}_{\sim \mathbb{P}(\b{h},\b{v})}[.]$, however, has nothing to do with the observation and is merely about the RBM model. 

\subsection{Contrastive Divergence}\label{section_contrastive_divergence}

According to Eq. (\ref{equation_RBM_MLE_derivative_first_term}), the conditional expectation used in Eq. (\ref{equation_h_hat}) includes one summation. Moreover, according to Eq. (\ref{equation_RBM_MLE_derivative_second_term}), the joint expectations used in Eqs. (\ref{equation_RBM_derivative_W}), (\ref{equation_RBM_derivative_b}), and (\ref{equation_RBM_derivative_c}) contain two summations. This double-summation makes computation of the joint expectation intractable because it sums over all possible values for both hidden and visible units. Therefore, exact computation of MLE is hard and we should approximate it. One way to approximate computation of joint expectations in MLE is \textit{contrastive divergence} \cite{hinton2002training}. 
Contrastive divergence improves the efficiency and reduces the variance of estimation in RBM \cite{hinton2002training,welling2004exponential}. 

The idea of contrastive divergence is as follows. First, we obtain a point $\widetilde{\b{v}}$ using Gibbs sampling starting from the observation $\b{v}_i$ (see Section \ref{section_RBM_Gibbs_sampling} for Gibbs sampling in RBM). Then, we compute expectation by using only that one point $\widetilde{\b{v}}$. 
The intuitive reason for why contrastive divergence works is explained in the following. 
We need to minimize the gradients to find the solution of MLE. In the joint expectations in Eqs. (\ref{equation_RBM_derivative_W}), (\ref{equation_RBM_derivative_b}), and (\ref{equation_RBM_derivative_c}), rather than considering all possible values of observations, contrastive divergence considers only one of the data points (observations). If this observation is a wrong belief which we do not wish to see in generation of observations by RBM, contrastive divergence is performing a task which is called \textit{negative sampling} \cite{hinton2012practical}. In negative sampling, we say rather than training the model to not generate all wrong observations, we train it iteratively but less ambitiously in every iteration. Each iteration tries to teach the model to not generate only one of the wrong outputs. Gradually, the model learns to generate correct observations by avoiding to generate these negative samples. 

Let $\widetilde{\b{h}} = [\widetilde{h}_1, \dots, \widetilde{h}_m]^\top$ be the corresponding sampled $\b{h}$ to $\widetilde{\b{v}} = [\widetilde{v}_1, \dots, \widetilde{v}_m]^\top$ in Gibbs sampling. According to the above explanations, contrastive divergence approximates the joint expectation in the derivative of log-likelihood, Eq. (\ref{equation_RBM_loglikelihood_derivative}), by Monte-Carlo approximation \cite{ghojogh2020sampling} evaluated at $\widetilde{\b{v}}_i$ and $\widetilde{\b{h}}_i$ for the $i$-th observation and hidden units where $\widetilde{\b{v}}_i$ and $\widetilde{\b{h}}_i$ are found by Gibbs sampling. Hence:
\begin{align}\label{equation_contrastive_divergence_MC_approximation}
&\mathbb{E}_{\sim \mathbb{P}(\b{h},\b{v})}[\nabla_\theta (-E(\b{v}, \b{h}))] \approx \nonumber\\
&~~~~~~~~~~~~~~~~~~~ \frac{1}{n} \sum_{i=1}^n \nabla_\theta (-E(\b{v}_i, \b{h}_i)) \Big|_{\b{v}_i=\widetilde{\b{v}}_i, \b{h}_i=\widetilde{\b{h}}_i}.
\end{align}
Experiments have shown that a small number of iterations in Gibbs sampling suffice for contrastive divergence. Paper \cite{hinton2002training} even uses one iteration of Gibbs sampling for this task. This small number of required iterations has the support of literature because Gibbs sampling is a special case of Metropolis-Hastings algorithms \cite{ghojogh2020sampling} which are fast \cite{dwivedi2018log}. 

\SetAlCapSkip{0.5em}
\IncMargin{0.8em}
\begin{algorithm2e}[!t]
\DontPrintSemicolon
    \textbf{Input: } training data $\{\b{x}_i\}_{i=1}^n$\;
    Randomly initialize $\b{W}, \b{b}, \b{c}$\;
    \While{not converged}{
        Sample a mini-batch $\{\b{v}_1, \dots, \b{v}_m\}$ from training dataset $\{\b{x}_i\}_{i=1}^n$ (n.b. we may set $m=n$)\;
        // Gibbs sampling for each data point: \;
        Initialize $\widehat{\b{v}}_i^{(0)} \gets \b{v}_i$ for all $i \in \{1, \dots, m\}$\;
        \For{$i$ from $1$ to $m$}{
            Algorithm \ref{algorithm_RBM_Gibbs_sampling} $\gets \widehat{\b{v}}_i^{(0)}$ \;
            $\{h_i\}_{i=1}^p, \{v_i\}_{i=1}^d \gets $ Last iteration of Algorithm \ref{algorithm_RBM_Gibbs_sampling}\;
            $\widetilde{\b{h}}_i \gets [h_1, \dots, h_p]^\top$\;
            $\widetilde{\b{v}}_i \gets [v_1, \dots, v_d]^\top$\;
            $\widehat{\b{h}}_i \gets \mathbb{E}_{\sim \mathbb{P}(\b{h}|\b{v}_i)}[\b{h}]$\;
        }
        // gradients:\;
        $\nabla_{\b{W}} \ell(\theta) \gets \sum_{i=1}^m \b{v}_i \widehat{\b{h}}_i^\top -\sum_{i=1}^m \widetilde{\b{h}}_i \widetilde{\b{v}}_i^\top$\;
        $\nabla_{\b{b}} \ell(\theta) \gets \sum_{i=1}^m \b{v}_i -\sum_{i=1}^m \widetilde{\b{v}}_i$\;
        $\nabla_{\b{c}} \ell(\theta) \gets \sum_{i=1}^m \widehat{\b{h}}_i -\sum_{i=1}^m \widetilde{\b{h}}_i$\;
        // gradient descent for updating solution: \\
        $\b{W} \gets \b{W} - \eta \nabla_{\b{W}} \ell(\theta)$\;
        $\b{b} \gets \b{b} - \eta \nabla_{\b{b}} \ell(\theta)$\;
        $\b{c} \gets \b{c} - \eta \nabla_{\b{c}} \ell(\theta)$\;
    }
    \textbf{Return} $\b{W}, \b{b}, \b{c}$\;
\caption{Training RBM using contrastive divergence}\label{algorithm_RBM_training_using_contrastive_divergence}
\end{algorithm2e}
\DecMargin{0.8em}

By the approximation in Eq. (\ref{equation_contrastive_divergence_MC_approximation}), the Eqs. (\ref{equation_RBM_derivative_W}), (\ref{equation_RBM_derivative_b}), and (\ref{equation_RBM_derivative_c}) become:
\begin{align}
& \nabla_{\b{W}} \ell(\theta) = \sum_{i=1}^n \b{v}_i \widehat{\b{h}}_i^\top - \sum_{i=1}^n \widetilde{\b{v}}_i \widetilde{\b{h}}_i^\top, \label{equation_RBM_derivative_W_2} \\
& \nabla_{\b{b}} \ell(\theta) = \sum_{i=1}^n \b{v}_i - \sum_{i=1}^n \widetilde{\b{v}}_i, \label{equation_RBM_derivative_b_2} \\
& \nabla_{\b{c}} \ell(\theta) = \sum_{i=1}^n \widehat{\b{h}}_i - \sum_{i=1}^n \widetilde{\b{h}}_i. \label{equation_RBM_derivative_c_2}
\end{align}
These equations make sense because when the observation variable and hidden variable given the observation variable become equal to the approximations by Gibbs sampling, the gradient should be zero and the training should stop. 
Note that some works in the literature restate Eqs. (\ref{equation_RBM_derivative_W_2}), (\ref{equation_RBM_derivative_b_2}), and (\ref{equation_RBM_derivative_c_2}) as \cite{hinton2002training,hinton2012practical,taylor2007modeling}:
\begin{align}
& \forall i,j\!: \, \nabla_{w_{ij}} \ell(\theta) = \langle v_i h_j \rangle_\text{data} - \langle v_i h_j \rangle_\text{recon.}, \label{equation_RBM_derivative_W_3} \\
& \forall i\!: \quad \nabla_{b_i} \ell(\theta) = \langle v_i \rangle_\text{data} - \langle v_i \rangle_\text{recon.}, \label{equation_RBM_derivative_b_3} \\
& \forall j\!: \quad \nabla_{c_j} \ell(\theta) = \langle h_j \rangle_\text{data} - \langle h_j \rangle_\text{recon.}, \label{equation_RBM_derivative_c_3}
\end{align}
where $\langle . \rangle_\text{data}$ and $\langle . \rangle_\text{recon.}$ denote expectation over data and reconstruction of data, respectively. 

The training algorithm of RBM, using contrastive divergence, can be found in Algorithm \ref{algorithm_RBM_training_using_contrastive_divergence}. In this algorithm, we are using mini-batch gradient descent with the batch size $m$. If training dataset is not large, one can set $m=n$ to have gradient descent. This algorithm is iterative until convergence where, in every iteration, a mini-batch is sampled where we have an observation $\b{v}_i \in \mathbb{R}^d$ and a hidden variable $\b{h}_i \in \mathbb{R}^p$ for every $i$-th training data point. For every data point, we apply Gibbs sampling as shown in Algorithm \ref{algorithm_RBM_Gibbs_sampling}. After Gibbs sampling, gradients are calculated by Eqs. (\ref{equation_RBM_derivative_W_2}), (\ref{equation_RBM_derivative_b_2}), and (\ref{equation_RBM_derivative_c_2}) and then the variables are updated using a gradient descent step.

\subsection{Boltzmann Machine}

So far, we introduced and explained RBM. BM has more links compared to RBM \cite{salakhutdinov2009deep}. 
Here, in parentheses, we briefly introduce training of BM.
Its structure is depicted in Fig. \ref{figure_BM_and_RBM}. As was explained in Section \ref{section_RBM_structure}, BM has additional links $\b{L} = [l_{ij}] \in \mathbb{R}^{d \times d}$ and $\b{J} = [j_{ij}] \in \mathbb{R}^{p \times p}$. 
The weights $\b{W} \in \mathbb{R}^{d \times p}$ and biases $\b{b} \in \mathbb{R}^d$ and $\b{c} \in \mathbb{R}^p$ are trained by gradient descent using the gradients in Eqs. (\ref{equation_RBM_derivative_W_2}), (\ref{equation_RBM_derivative_b_2}), and (\ref{equation_RBM_derivative_c_2}). The additional weights $\b{L}$ and $\b{J}$ are updated similarly using the following gradients \cite{salakhutdinov2009deep}:
\begin{align}
& \nabla_{\b{L}} \ell(\theta) = \sum_{i=1}^n \b{v}_i \b{v}_i^\top - \sum_{i=1}^n \widetilde{\b{v}}_i \widetilde{\b{v}}_i^\top, \label{equation_BM_derivative_L} \\
& \nabla_{\b{J}} \ell(\theta) = \sum_{i=1}^n \mathbb{E}_{\sim \mathbb{P}(\b{h}|\b{v}_i)}[\b{h} \b{h}^\top] - \sum_{i=1}^n \widetilde{\b{h}}_i \widetilde{\b{h}}_i^\top. \label{equation_BM_derivative_J} 
\end{align}
These equations can be restated as:
\begin{align}
& \forall i,j\!: \, \nabla_{l_{ij}} \ell(\theta) = \langle v_i v_j \rangle_\text{data} - \langle v_i v_j \rangle_\text{recon.}, \label{equation_BM_derivative_L_2} \\
& \forall i,j\!: \, \nabla_{j_{ij}} \ell(\theta) = \langle h_i h_j \rangle_\text{data} - \langle h_i h_j \rangle_\text{recon.}, \label{equation_BM_derivative_J_2}
\end{align}
where $\langle . \rangle_\text{data}$ and $\langle . \rangle_\text{recon.}$ denote expectation over data and reconstruction of data, respectively.

\section{Distributions of Visible and Hidden Variables}\label{section_distributions_of_variables}

\subsection{Modeling with Exponential Family Distributions}

According to Proposition \ref{proposition_RBM_conditional_independence}, the units $\b{v} \in \mathbb{R}^d$ and $\b{h} \in \mathbb{R}^p$ have conditional independence so their distribution is the product of each conditional distribution.
We can choose distributions from the exponential family of distributions for the visible and hidden variables \cite{welling2004exponential}: 
\begin{align}
& \mathbb{P}(\b{v}) = \prod_{i=1}^d r_i(v_i) \exp\Big( \sum_{a} \theta_{ia}\, f_{ia}(v_i) - A_i(\{\theta_{ia}\}) \Big), \\
& \mathbb{P}(\b{h}) = \prod_{j=1}^p s_j(h_j) \exp\Big( \sum_{b} \lambda_{jb}\, g_{jb}(h_j) - B_j(\{\lambda_{jb}\}) \Big),
\end{align}
where $\{f_{ia}(v_i), g_{jb}(h_j)\}$ are the sufficient statistics, $\{\theta_i, \lambda_j\}$ are the canonical parameters of the models, $\{A_i, B_j\}$ are the log-normalization factors, and $\{r_i(v_i), s_j(h_j)\}$ are the normalization factors which are some additional features multiplied by some constants. 
We can ignore the log-normalization factors because they are hard to compute. 

For the joint distribution of visible and hidden variables, we should introduce a quadratic term for their cross-interaction \cite{welling2004exponential}: 
\begin{align}\label{equation_RBM_exponentialFamily_joint}
\mathbb{P}(\b{v}, \b{h}) &\propto \exp\Big( \sum_{i=1}^d \sum_{a} \theta_{ia}\, f_{ia}(v_i) \nonumber \\
&+ \sum_{j=1}^p \sum_{b} \lambda_{jb}\, g_{jb}(h_j) \nonumber \\
& + \sum_{i=1}^d \sum_{j=1}^p \sum_{a} \sum_{b} \b{W}_{ia}^{jb}\, f_{ia}(v_i)\, g_{jb}(h_j) \Big).
\end{align}
According to Proposition \ref{proposition_RBM_conditional_independence}, the visible and hidden units have conditional independence. Therefore, the conditional distributions can be written as multiplication of exponential family distributions \cite{welling2004exponential}:
\begin{align}
& \mathbb{P}(\b{v}|\b{h}) = \prod_{i=1}^d \exp\Big( \sum_{a} \widehat{\theta}_{ia}\, f_{ia}(v_i) - A_i(\{\widehat{\theta}_{ia}\}) \Big), \\
& \mathbb{P}(\b{h}|\b{v}) = \prod_{j=1}^p \exp\Big(\! \sum_{b} \widehat{\lambda}_{jb}\, g_{jb}(h_j) - B_j(\{\widehat{\lambda}_{jb}\}) \Big),
\end{align}
where:
\begin{align}
& \widehat{\theta}_{ia} := \theta_{ia} + \sum_{j=1}^p \sum_b \b{W}_{ia}^{jb}\, g_{jb}(h_j), \\
& \widehat{\lambda}_{jb} := \lambda_{jb} + \sum_{i=1}^d \sum_a \b{W}_{ia}^{jb}\, f_{ia}(v_i).
\end{align}
Therefore, we can choose one of the distributions in the exponential family for the conditional distributions of visible and hidden variables. In the following, we introduce different cases where the units can have either discrete or continuous values. In all cases, the distributions are from exponential families. 

\subsection{Binary States}\label{sectiobn_RBM_binary_states}

The hidden and visible variables can have discrete number of values, also called states. Most often, inspired by the Hopfield network, BM and RBM have binary states. In this case, the hidden and visible units can have binary states, i.e., $v_i, h_j \in \{0, 1\}, \forall i,j$. Hence, we can say:
\begin{align}\label{equation_binaryStates_prob_h_j_one_and_v}
\mathbb{P}(h_j=1 | \b{v}) &= \frac{\mathbb{P}(h_j=1, \b{v})}{\mathbb{P}(h_j=0, \b{v}) + \mathbb{P}(h_j=1, \b{v})}.
\end{align}
In binary states, the joint probability in Eq. (\ref{equation_prob_h_j_and_v}) is simplified to:
\begin{align*}
&\mathbb{P}(h_j = 0, \b{v}) = \exp( c_j \times 0 + \b{v}^\top \b{W}_{:j} \times 0) \\
&~~~~~~~~~~~~~~~~~~~~~ = \exp(0) = 1, \\
&\mathbb{P}(h_j = 1, \b{v}) = \exp( c_j \times 1 + \b{v}^\top \b{W}_{:j} \times 1) \\
&~~~~~~~~~~~~~~~~~~~~~ = \exp( c_j + \b{v}^\top \b{W}_{:j}).
\end{align*}
Hence, Eq. (\ref{equation_binaryStates_prob_h_j_one_and_v}) becomes:
\begin{align}
&\mathbb{P}(h_j=1 | \b{v}) = \frac{\exp( c_j + \b{v}^\top \b{W}_{:j})}{1 + \exp( c_j + \b{v}^\top \b{W}_{:j})} \nonumber \\
&= \frac{1}{1 + \exp\big(\!-(c_j + \b{v}^\top \b{W}_{:j})\big)} = \sigma(c_j + \b{v}^\top \b{W}_{:j}), \label{equation_binary_states_h_given_v}
\end{align}
where:
\begin{align*}
\sigma(x) := \frac{\exp(x)}{1 + \exp(x)} = \frac{1}{1 + \exp(-x)},
\end{align*}
is the sigmoid (or logistic) function and $\b{W}_{:j} \in \mathbb{R}^d$ denotes the $j$-th column of matrix $\b{W}$.
If the visible units also have binary states, we will similarly have:
\begin{align}\label{equation_binary_states_v_given_h}
&\mathbb{P}(v_i=1 | \b{h}) = \sigma(b_i + \b{W}_{i:}\, \b{h}),
\end{align}
where $\b{W}_{i:} \in \mathbb{R}^p$ denotes the $i$-th row of matrix $\b{W}$.
As we have only two states $\{0,1\}$, from Eqs. (\ref{equation_binary_states_h_given_v}) and (\ref{equation_binary_states_v_given_h}), we have:
\begin{align}
& \mathbb{P}(h_j | \b{v}) = \sigma(c_j + \b{v}^\top \b{W}_{:j}) = \sigma(c_j + \sum_{i=1}^d v_i\, w_{ij}), \label{equation_binary_states_h_given_v_2}\\
& \mathbb{P}(v_i | \b{h}) = \sigma(b_i + \b{W}_{i:}\, \b{h}) = \sigma(b_i + \sum_{j=1}^p w_{ij}\, h_j). \label{equation_binary_states_v_given_h_2}
\end{align}
According to Proposition \ref{proposition_RBM_conditional_independence}, the units have conditional independence so their distribution is the product of each conditional distribution:
\begin{align}
&\mathbb{P}(\b{h} | \b{v}) = \prod_{j=1}^p \mathbb{P}(h_j | \b{v}) = \prod_{j=1}^p \sigma(c_j + \b{v}^\top \b{W}_{:j}), \label{equation_binary_states_h_given_v_3} \\ &\mathbb{P}(\b{v} | \b{h}) = \prod_{i=1}^d \mathbb{P}(v_i | \b{h}) = \prod_{i=1}^d \sigma(b_i + \b{W}_{i:}\, \b{h}). \label{equation_binary_states_v_given_h_3}
\end{align}
Therefore, in Gibbs sampling of Algorithm \ref{algorithm_RBM_Gibbs_sampling}, we sample from the distributions of Eqs. (\ref{equation_binary_states_h_given_v_2}) and (\ref{equation_binary_states_v_given_h_2}).
Note that the sigmoid function is between zero and one so we can use the uniform distribution $u \sim U[0,1]$ for sampling from it. This was explained in Section \ref{section_RBM_Gibbs_sampling}. 
Moreover, for binary states, we have $\mathbb{E}_{\sim \mathbb{P}(h_j|\b{v})}[h_j] = \sigma(c_j + \b{v}^\top \b{W}_{:j})$. Hence, if we apply the sigmoid function element-wise on elements of $\widehat{\b{h}}_i \in \mathbb{R}^p$, we can have this for Eq. (\ref{equation_h_hat}) in training binary-state RBM:
\begin{align}
\widehat{\b{h}}_i = \mathbb{E}_{\sim \mathbb{P}(\b{h}|\b{v}_i)}[\b{h}] = \sigma(\b{c} + \b{v}_i^\top \b{W}).
\end{align}
This equation is also used in Algorithm \ref{algorithm_RBM_training_using_contrastive_divergence}.

\subsection{Continuous Values}

In some cases, we set the hidden units to have continuous values as continuous representations for the visible unit. According to the definition of conditional probability, we have:
\begin{align}\label{equation_continuousStates_prob_h_j_one_and_v}
\mathbb{P}(v_i=1 | \b{h}) &= \frac{\mathbb{P}(v_i=1, \b{h})}{\sum_{v_i \in \mathbb{R}^d} \mathbb{P}(\b{h}, v_i)}.
\end{align}
According to Eq. (\ref{equation_prob_h_and_v_i}), we have:
\begin{align}
\mathbb{P}(\b{h}, v_i=1) &= \exp( b_i \times 1 + 1 \times \b{W}_{i:}\, \b{h}) \nonumber \\
&= \exp( b_i + \b{W}_{i:}\, \b{h}). 
\end{align}
Hence, Eq. (\ref{equation_continuousStates_prob_h_j_one_and_v}) becomes:
\begin{align}
\mathbb{P}(v_i=1 | \b{h}) &= \frac{\exp( b_i + \b{W}_{i:}\, \b{h})}{\sum_{v_i \in \mathbb{R}^d} \mathbb{P}(\b{h}, v_i)}.
\end{align}
This is a softmax function which can approximate a Gaussian (normal) distribution and it sums to one. Therefore, we can write it as the normal distribution with variance one:
\begin{align}
\mathbb{P}(v_i| \b{h}) &= \mathcal{N}(b_i + \b{W}_{i:}\, \b{h}, 1) = \mathcal{N}(b_i + \sum_{j=1}^p w_{ij} h_j, 1).
\end{align}
According to Proposition \ref{proposition_RBM_conditional_independence}, the units have conditional independence so their distribution is the product of each conditional distribution:
\begin{align}
\mathbb{P}(\b{v}| \b{h}) = \prod_{i=1}^d \mathbb{P}(v_i| \b{h}) &= \prod_{i=1}^d \mathcal{N}(b_i + \b{W}_{i:}\, \b{h}, 1).
\end{align}
Usually, when the hidden units have continuous values, the visible units have binary states \cite{welling2004exponential,mohamed2010phone}. In this case, the conditional distribution $\mathbb{P}(\b{h}| \b{v})$ is obtained by Eq. (\ref{equation_binary_states_h_given_v_3}). If the visible units have continuous values, their distribution can be similarly calculated as:
\begin{align}
\mathbb{P}(\b{h}| \b{v}) = \prod_{j=1}^p \mathbb{P}(h_j| \b{v}) &= \prod_{j=1}^p \mathcal{N}(c_j + \b{v}^\top \b{W}_{:j}, 1).
\end{align}
These normal distributions can be used for sampling in Gibbs sampling of Algorithm \ref{algorithm_RBM_Gibbs_sampling}. 
In this case, Eq. (\ref{equation_h_hat}), used in Algorithm \ref{algorithm_RBM_training_using_contrastive_divergence} for training RBM, is:
\begin{align}
\widehat{\b{h}}_i = \mathbb{E}_{\sim \mathbb{P}(\b{h}|\b{v}_i)}[\b{h}] = \mathcal{N}(\b{c} + \b{v}_i^\top \b{W}, \b{I}_{p \times p}),
\end{align}
where $\b{I}$ denotes the identity matrix. 

\subsection{Discrete Poisson States}

In some cases, the units have discrete states but with more than two values which was discussed in Section \ref{sectiobn_RBM_binary_states}. In this case, we can use the well-known Poisson distribution for discrete random variables:
\begin{align*}
\text{Ps}(t, \lambda) = \frac{e^{-\lambda} \lambda^t}{t!}.
\end{align*}
Assume every visible unit can have a value $t \in \{0,1,2,3,\dots\}$.
If we consider the conditional Poisson distribution for the visible units, we have \cite{salakhutdinov2009semantic}:
\begin{align}
\mathbb{P}(v_i=t | \b{h}) = \text{Ps}(t, \frac{\exp(b_i + \b{W}_{i:}\, \b{h})}{\sum_{k=1}^d \exp(b_k + \b{W}_{k:}\, \b{h})}).
\end{align}
Similarly, if we have discrete states for the hidden units, we can have:
\begin{align}
\mathbb{P}(h_j=t | \b{v}) = \text{Ps}(t, \frac{\exp(c_j + \b{v}^\top \b{W}_{:j})}{\sum_{k=1}^p \exp(c_k + \b{v}^\top \b{W}_{:k})}),
\end{align}
These Poisson distributions can be used for sampling in Gibbs sampling of Algorithm \ref{algorithm_RBM_Gibbs_sampling}. 
In this case, Eq. (\ref{equation_h_hat}), used in Algorithm \ref{algorithm_RBM_training_using_contrastive_divergence} for training RBM, can be calculated using a multivariate Poisson distribution \cite{edwards1962multivariate}. 
It is noteworthy that RBM has been used for \textit{semantic hashing} where the hidden variables are used as a hashing representation of data \cite{salakhutdinov2009semantic}. Semantic hashing uses Poisson distribution and sigmoid function for the conditional visible and hidden variables, respectively. 







\section{Conditional Restricted Boltzmann Machine}\label{section_conditional_RBM}








\begin{figure*}[!t]
\centering
\includegraphics[width=5.5in]{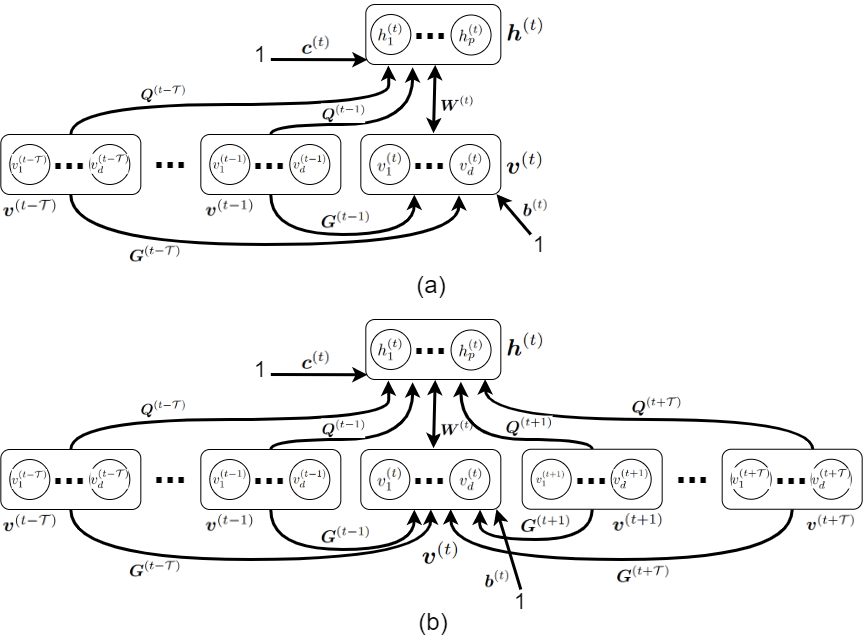}
\caption{The structures of (a) a conditional Boltzmann machine and (b) an interpolating conditional restricted Boltzmann machine. Note that each arrow in this figure represents a set of links.}
\label{figure_CRBM_and_ICRBM}
\end{figure*}

If data are a time-series, RBM does not include its temporal (time) information. In other words, RBM is suitable for static data. Conditional RBM (CRBM), proposed in \cite{taylor2007modeling}, incorporates the temporal information into the configuration of RBM. 
It considers visible variables of previous time steps as conditional variables.
CRBM adds two sets of directed links to RBM. The first set of directed links is the autoregressive links from the past $\mathcal{T}_1$ visible units to the visible units of current time step. The second set of directed links is the links from the past $\tau_2$ visible units to the hidden units of the current time step. In general, $\mathcal{T}_1$ is not necessarily equal to $\mathcal{T}_2$ but for simplicity, we usually set $\mathcal{T}_1 = \mathcal{T}_2 = \mathcal{T}$ \cite{taylor2007modeling}. We denote the links from the visible units at time $t-\tau$ to visible units at current time by $\b{G}^{(t-\tau)} = [g_{ij}] \in \mathbb{R}^{d \times d}$. We also denote the links from the visible units at time $t-\tau$ to hidden units at current time by $\b{Q}^{(t-\tau)} = [q_{ij}] \in \mathbb{R}^{d \times p}$.
The structure of CRBM is shown in Fig. \ref{figure_CRBM_and_ICRBM}. Note that each arrow in this figure is a set of links representing a matrix or vector of weights.

The updating rule for the weights $\b{W}$ and biases $\b{b}$ and $\b{c}$ are the same as in Eqs. (\ref{equation_RBM_derivative_W_2}), (\ref{equation_RBM_derivative_b_2}), and (\ref{equation_RBM_derivative_c_2}). We consider the directed links from the previous visible units to current visible units and current hidden units as dynamically changing biases. Recall that for updating the biases $\b{b}$ and $\b{c}$, we used Eqs. (\ref{equation_RBM_derivative_b_2}) and (\ref{equation_RBM_derivative_c_2}). 
Similarly, for updating the added links from $\tau$ previous time steps, we have:
\begin{align}
& \forall i\!: \mathbb{R}^d \ni \nabla_{\b{G}_{i:}^{(t-\tau)}} \ell(\theta) = v_i^{(t-\tau)} \Big(\sum_{k=1}^n \b{v}_k^{(t)} - \sum_{k=1}^n \widetilde{\b{v}}_k^{(t)}\Big), \label{equation_CRBM_derivative_b_2} \\
& \forall i\!: \mathbb{R}^p \ni \nabla_{\b{Q}_{i:}^{(t-\tau)}} \ell(\theta) = v_i^{(t-\tau)} \Big(\sum_{k=1}^n \widehat{\b{h}}_k - \sum_{k=1}^n \widetilde{\b{h}}_k\Big), \label{equation_CRBM_derivative_c_2}
\end{align}
where $i \in \{1, \dots, d\}$, $\tau \in \{1, 2, \dots, \mathcal{T}\}$, $\b{G}_{i:}^{(t-\tau)}$ denotes the $i$-th row of $\b{G}^{(t-\tau)}$, and $\b{Q}_{i:}^{(t-\tau)}$ denotes the $i$-th row of $\b{Q}^{(t-\tau)}$.
These equations are similar to Eqs. (\ref{equation_RBM_derivative_b_2}) and (\ref{equation_RBM_derivative_c_2}) but they are multiplied by the visible values in previous time steps. This is because RBM multiplies biases with one (see Fig. \ref{figure_BM_and_RBM}) while these newly introduced biases have previous visible values instead of one (see Fig. \ref{figure_CRBM_and_ICRBM}). 
These equations can be restated as \cite{taylor2007modeling}:
\begin{align}
& \forall i,j\!: \, \nabla_{g_{ij}^{(t-\tau)}} \ell(\theta) = v_i^{(t-\tau)} \Big(\langle v_i^{(t)} \rangle_\text{data} - \langle v_i^{(t)} \rangle_\text{recon.}\Big), \label{equation_CRBM_derivative_b_3} \\
& \forall i,j\!: \, \nabla_{q_{ij}^{(t-\tau)}} \ell(\theta) = v_i^{(t-\tau)} \Big(\langle h_j^{(t)} \rangle_\text{data} - \langle h_j^{(t)} \rangle_\text{recon.}\Big), \label{equation_CRBM_derivative_c_3}
\end{align}
where $\langle . \rangle_\text{data}$ and $\langle . \rangle_\text{recon.}$ denote expectation over data and reconstruction of data, respectively. 
In addition to the weights and biases of RBM, the additional links are learned by gradient descent using the above gradients. 
Algorithm \ref{algorithm_RBM_training_using_contrastive_divergence} can be used for training CRBM if learning the added links is also included in the algorithm.

Interpolating CRBM (ICRBM) \cite{mohamed2010phone} is an improvement over the CRBM where some links have been added from visible variables in the future.
Figure \ref{figure_CRBM_and_ICRBM} depicts the structure of ICRBM.
Its training and formulation are similar but we do not cover its theory in this paper for the sake of brevity. 
Note that CRBM has been used in various time-series applications such as action recognition \cite{taylor2007modeling} and acoustics \cite{mohamed2010phone}.

\section{Deep Belief Network}\label{section_DBN}

\subsection{Stacking RBM Models}

We can train a neural network using RBM training \cite{hinton2006reducing,hinton2006fast}. Training a neural network using RBM training can result in very good initialization of weights for training network using backpropagation. Before the development of ReLU \cite{glorot2011deep} and dropout \cite{srivastava2014dropout}, multilayer perceptron networks could not become deep for the problem of vanishing gradients. This was because random initial weights were not suitable enough for starting optimization in backpropagation, especially in deep networks. Therefore, a method was proposed for pre-training neural networks which initializes network to a suitable set of weights and then the pre-trained weights are fine-tuned using backpropagation \cite{hinton2006reducing,hinton2006fast}.

A neural network consists of several layers. Let $\ell$ denote the number of layers, where the first layer gets the input data, and let $p_\ell$ be the number of neurons in the $\ell$-th layer. By convention, we have $p_1 = d$. We can consider every two successive layers as one RBM. This is shown in Fig. \ref{figure_DBN}. We start from the first pair of layers as an RBM and we introduce training dataset $\{\b{x}_i \in \mathbb{R}^d\}_{i=1}^n$ as the visible variable $\{\b{v}_i\}_{i=1}^n$ of the first pair of layers. We train the weights and biases of this first layer as an RBM using Algorithm \ref{algorithm_RBM_training_using_contrastive_divergence}. After training this RBM, we generate $n$ $p_2$-dimensional hidden variables using Gibbs sampling in Algorithm \ref{algorithm_RBM_Gibbs_sampling}. Now, we consider the hidden variables of the first RBM as the visible variables for the second RBM (the second pair of layers). Again, this RBM is trained by Algorithm \ref{algorithm_RBM_training_using_contrastive_divergence} and, then, hidden variables are generated using Gibbs sampling in Algorithm \ref{algorithm_RBM_Gibbs_sampling}. This procedure is repeated until all pairs of layers are trained using RBM training. 
This layer-wise training of neural network has a greedy approach \cite{bengio2007greedy}.
This greedy training of layers prepares good initialized weights and biases for the whole neural network. After this initialization, we can fine-tune the weights and biases using backpropagation \cite{rumelhart1986learning}.

The explained training algorithm was first proposed in \cite{hinton2006reducing,hinton2006fast} and was used for dimensionality reduction. 
By increasing $\ell$ to any large number, the network becomes large and deep. As layers are trained one by one as RBM models, we can make the network as deep as we want without being worried for vanishing gradients because weights are initialized well for backpropagation. As this network can get deep and is pre-trained by belief propagation (RBM training), it is referred to as the \textit{Deep Belief Network} (DBN) \cite{hinton2006fast,hinton2009deep}. DBN can be seen as a stack of RBM models. 
The pre-training of a DBN using RBM training is depicted in Fig. \ref{figure_DBN}.
This algorithm is summarized in Algorithm \ref{algorithm_training_deep_belief_network}.
In this algorithm, $\b{W}_l \in \mathbb{R}^{p_l \times p_{l+1}}$ denotes the weights connecting layer $l$ to layer $(l+1)$ and $\b{b}_l \in \mathbb{R}^{p_l}$ denotes the biases of the layer $l$. Note that as weights are between every two layers, the sets of weights are $\{\b{W}_l\}_{l=1}^{\ell-1}$.

Note that pre-training of DBN is an unsupervised task because RBM training is unsupervised. Fine-tuning of DBN can be either unsupervised or supervised depending on the loss function for backpropagation. 
If the DBN is an autoencoder with a low-dimensional middle layer in the network, both its pre-training and fine-tuning stages are unsupervised because the loss function of backpropagation is also a mean squared error. This DBN autoencoder can learn a low-dimensional embedding or representation of data and can be used for dimensionality reduction \cite{hinton2006reducing}. The DBN autoencoder has also been used for hashing \cite{salakhutdinov2009semantic}. The structure of this network is depicted in Fig. \ref{figure_DBN_autoencoder}.

\SetAlCapSkip{0.5em}
\IncMargin{0.8em}
\begin{algorithm2e}[!t]
\DontPrintSemicolon
    \textbf{Input}: training data $\{\b{x}_i\}_{i=1}^n$\;
    // pre-training:\;
    \For{$l$ from $1$ to $\ell-1$}{
        \uIf{$l = 1$}{
            $\{\b{v}_i\}_{i=1}^n \gets \{\b{x}_i\}_{i=1}^n$\;  
        }
        \Else{
            // generate $n$ hidden variables of previous RBM:\;
            $\{\b{h}_i\}_{i=1}^n \gets$ Algorithm \ref{algorithm_RBM_Gibbs_sampling} for $(l-1)$-th RBM $\gets \{\b{v}_i\}_{i=1}^n$\;
            $\{\b{v}_i\}_{i=1}^n \gets \{\b{h}_i\}_{i=1}^n$\;
        }
        $\b{W}_l, \b{b}_l, \b{b}_{l+1} \gets$ Algorithm \ref{algorithm_RBM_training_using_contrastive_divergence} for $l$-th RBM $\gets \{\b{v}_i\}_{i=1}^n$\;    
    }
    // fine-tuning using backpropagation:\;
    Initialize network with weights $\{\b{W}_l\}_{l=1}^{\ell-1}$ and biases $\{\b{b}_l\}_{l=2}^{\ell}$.\;
    $\{\b{W}_l\}_{l=1}^{\ell-1}, \{\b{b}_l\}_{l=1}^{\ell} \gets$ Backpropagate the error of loss fro several epochs.\;
\caption{Training a deep belief network}\label{algorithm_training_deep_belief_network}
\end{algorithm2e}
\DecMargin{0.8em}




\subsection{Other Improvements over RBM and DBN}\label{section_other_improvements_RBM_DBN}

Some improvements over DBN are convolutional DBN \cite{krizhevsky2010convolutional} and use of DBN for hashing \cite{salakhutdinov2009semantic}.
Greedy training of DBN using RBM training has been used in training the t-SNE network for general degree of freedom \cite{van2009learning}. 
In addition to CRBM \cite{taylor2007modeling}, recurrent RBM \cite{sutskever2009recurrent} has been proposed to handle the temporal information of data. 
Also, note that there exist some other energy-based models in addition to BM; Helmholtz machine \cite{dayan1995helmholtz} is an example. 

There also exists Deep Boltzmann Machine (DBM) which is slightly different from DBN. For the sake of brevity, we do not cover it here and refer the interested reader to \cite{salakhutdinov2009deep}.
Various efficient training algorithms have been proposed for DBM \cite{salakhutdinov2010efficient,salakhutdinov2010learning,salakhutdinov2012efficient,srivastava2012multimodal,hinton2012better,montavon2012deep,goodfellow2013multi,srivastava2014multimodal,melchior2016center}. Some of the applications of DBM are document processing \cite{srivastava2013modeling}, face modeling \cite{nhan2015beyond}

\begin{figure}[!t]
\centering
\includegraphics[width=3in]{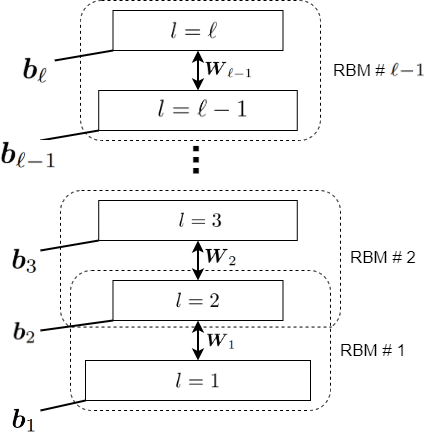}
\caption{Pre-training a deep belief network by considering every pair of layers as an RBM.}
\label{figure_DBN}
\end{figure}

\begin{figure}[!t]
\centering
\includegraphics[width=3in]{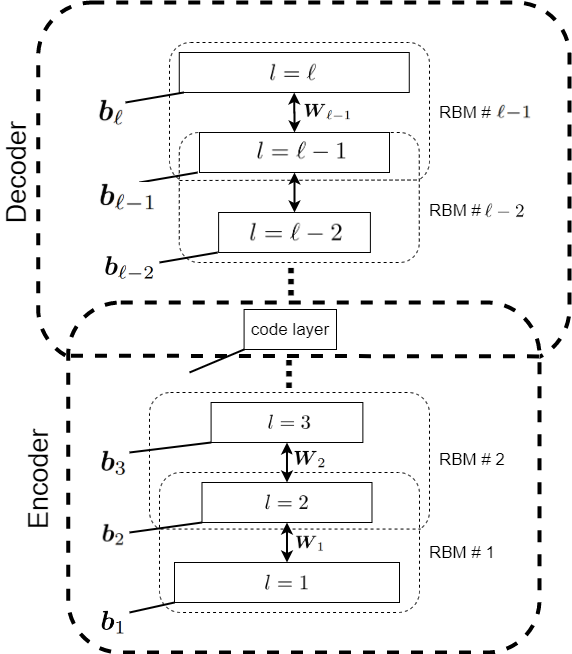}
\caption{A DBN autoencoder where the number of neurons in the corresponding layers of encoder and decoder are usually set to be equal. The coder layer is a low-dimensional embedding for representation of data.}
\label{figure_DBN_autoencoder}
\end{figure}

\section{Conclusion}\label{section_conclusion}

This was a tutorial paper on BM, RBM, and DBN. After some background, we covered the structure of BM and RBM, Gibbs sampling in RBM for generating visible and hidden variables, training RBM using contrastive divergence, and training BM. Then, we introduced various cases for states of visible and hidden units. Thereafter, CRBM and DBN were explained in detail. 

\section*{Acknowledgement}

The authors hugely thank Prof. Mehdi Molkaraie for his course which partly covered some materials on Ising model and statistical physics. 
Some of the materials in this tutorial paper have been covered by Prof. Ali Ghodsi's videos on YouTube. 


\bibliography{References}
\bibliographystyle{icml2016}

\end{document}